%% file: main.tex
\documentclass{ngsm2024} 

\newcommand{\A}{\mathbf{A}}
\newcommand{\B}{\mathbf{B}}
\newcommand{\C}{\mathbf{C}}
\newcommand{\D}{\mathbf{D}}
\usepackage{algorithm, algorithmic}
\usepackage{titlesec}
\titleformat{\paragraph}
{\normalfont\normalsize\bfseries}{\theparagraph}{1em}{}
\titlespacing*{\paragraph}
{0pt}{3.25ex plus 1ex minus .2ex}{1.5ex plus .2ex}
\usepackage{thmtools, thm-restate}

\usepackage[symbol]{footmisc}
\allowdisplaybreaks

\title[HiPPO-Prophecy: State-Space Models can Provably Learn Dynamical Systems in Context]{HiPPO-Prophecy: State-Space Models can Provably Learn Dynamical Systems in Context}

\optauthor{%
\Name{Federico Arangath Joseph}\textsuperscript{*} \Email{farangath@ethz.ch}\\
\Name{Kilian Haefeli}\textsuperscript{*} \Email{khaefeli@ethz.ch}\\
\Name{Noah Liniger}\textsuperscript{*} \Email{nliniger@ethz.ch}\\
\addr ETH Zurich\\
\Name{Caglar Gulcehre} \Email{caglar.gulcehre@epfl.ch}\\
\addr EPFL}

\newcounter{relctr} 
\everydisplay\expandafter{\the\everydisplay\setcounter{relctr}{0}} 

\AtBeginDocument{} 
\usepackage{parskip}
\usepackage{enumitem}  
\usepackage{amssymb}
\usepackage{amsfonts}  
\usepackage{microtype}
\usepackage{graphicx}
\usepackage{booktabs} 
\usepackage{jmlrutils} 
\begin{document}
\maketitle
\footnotetext[1]{Equal Contribution}
\begin{abstract}%
\input{sections/abstract}
\end{abstract}
\input{sections/introduction}
\input{sections/ssm_hippo_theory}
\input{sections/solving_dynamical_systems_with_ssm}
\input{sections/experiments}
\input{sections/conclusions}

\bibliography{main}
\newpage
\clearpage
\appendix
\input{sections/appendix}

\end{document}

%% file: sections/abstract.tex
This work explores the in-context learning capabilities of State Space Models (SSMs) and presents, to the best of our knowledge, the first theoretical explanation of a possible underlying mechanism. We introduce a novel weight construction for SSMs, enabling them to predict the next state of any dynamical system after observing previous states without parameter fine-tuning. 
This is accomplished by extending the HiPPO framework to demonstrate that continuous SSMs can approximate the derivative of any input signal. Specifically, we find an explicit weight construction for continuous SSMs and provide an asymptotic error bound on the derivative approximation. The discretization of this continuous SSM subsequently yields a discrete SSM that predicts the next state. Finally, we demonstrate the effectiveness of our parameterization empirically. This work should be an initial step toward understanding how sequence models based on SSMs learn in context.

%% file: sections/introduction.tex
\section{Introduction}
In-context learning (ICL) refers to a model's ability to solve tasks unseen during training, only based on information provided in context, without updating its weights. ICL has gained significant attention since \citet{brown2020language} demonstrated that transformer models \cite{vaswani2023attention} trained in large and diverse language corpora can learn in context without being explicitly trained for it. More specifically, they showed that given a sequence of input-output pairs from an unseen task, the model can predict the output corresponding to a new input. We will refer to this type of ICL as few-shot in-context learning to emphasize the presence of input-output pairs in-context. Subsequently, to \cite{brown2020language}, there has been a variety of empirical \cite{olsson2022incontext, kossen2024incontext, wei2023larger, garg2023transformers} as well as theoretical \cite{vonoswald2023transformers, akyürek2023learning, bai2023transformers, vladymyrov2024linear,zhang2023trained,wu2024pretraining, zhang2024incontext,mahankali2023step, ahn2023transformers} works studying the few-shot ICL capabilities of transformer models and the mechanisms underlying them.
Nonetheless, there still exists a gap between the studied few-shot ICL settings and ICL capabilities that emerge in modern sequence models, trained autoregressively on sequential data \cite{akyürek2024incontext, vonoswald2023uncovering}. Previous works by \cite{akyürek2024incontext, vonoswald2023uncovering} aim to close this gap by studying the ICL capabilities of autoregressively trained transformer models to predict the next value \smash{$f_{k+1}$} of an unseen sequence, when provided with values $f_{\leq k}$ in-context. We term this mode of learning from context as autoregressive ICL.\par
Concurrently to these studies, the development of deep state space models such as S4 by \citet{gu2022efficiently} has sparked a resurgence of recurrent sequence models, resulting in a family of models, which we refer to as generalized state space models (GSSMs) \cite{hasani2022liquid,gu2022parameterization, orvieto2023resurrecting, fu2023hungry,gu2023mamba, de2024griffin}. GSSMs offer a promising alternative to transformers by addressing two of their fundamental shortcomings, namely length generalization and quadratic computational cost (flops) with respect to the sequence length \cite{gu2023mamba}. Similarly to transformer-based models, GSSMs have empirically been shown to be capable of ICL \cite{lu2023structured, akyürek2024incontext, park2024mamba, grazzi2024mamba}. Despite their potential, our theoretical understanding of the mechanisms underlying ICL in GSSMs is still limited. To the best of our knowledge, this work provides the first explanation of how GSSMs can perform autoregressive ICL. For this we consider SSMs on the task of predicting the next state of any dynamical system, given a sequence of previous states.
Our contributions can be summarized as follows.
\begin{itemize}[leftmargin=*,itemsep=-3pt]
  \item Extending the HiPPO framework to show that SSMs can approximate the next state of any dynamical system up to first order from a sequence of previous states through an explicit weight construction. (\S \ref{sec:solving_dynamical_syemstems_with_SSMs})
  \item An asymptotic bound on the error incurred when approximating the derivative of the input signal with a continuous SSM parametrized with our proposed construction for the FouT basis. (\S \ref{sec:solving_dynamical_syemstems_with_SSMs})
  \item An experimental evaluation of our weight construction on different function classes, model sizes, and context lengths. (\S \ref{sec:experiments})
\end{itemize}

%% file: sections/ssm_hippo_theory.tex
\section{SSMs and HiPPO Theory}
\label{sec:ssm_and_hippo_theory}
SSMs map an input signal $u(t) \in \mathbb{R}$ to an output signal $y(t) \in \mathbb{R}$ via a hidden state $x(t) \in \mathbb{R}^N$ and take the following form:
\begin{equation}
    \begin{cases}
        \dot{x}(t) = \A x(t)+\B u(t) \\
        y(t) = \C x(t)+\D u(t)
    \end{cases}
\end{equation}
One key property of these models is their ability to memorize past information, $u(s)$ for $s\leq t$, 
in their hidden state $x(t)$. This capability was established in the HiPPO theory \cite{Gu2020Hippo,gu2022train}. HiPPO theory considers a Hilbert space spanned by the orthogonal basis functions $\{ p_n(t,s)\}_{n \geq 0}$, equipped with a measure $\omega(t,s)$ and an inner product $\langle f, g \rangle_\omega = \int_{-\infty}^t f(s)g(s) \omega(t,s) \mathrm{d}s$. The theory proposes a parametrization for  $\A$ and $\B$ such that the hidden state $x_n(t)$ represents the projection of the input signal $u_{\leq t}$ onto the basis $p_n(t,s)$ i.e. $x_n(t) = \langle u_{\leq t}, p_n \rangle_\omega$. One of the fundamental results in HiPPO theory is that in the limit of an infinite hidden state size $N \rightarrow \infty$ and for an appropriate choice of $\A$ and $\B$, it is possible to reconstruct the input signal up to time $t$ from the hidden state $x(t)$:
\begin{equation}
\label{eq:reconstruction}
    u(s) = \sum_{n=0}^{\infty} x_n(t) p_n(t,s) \ \ \forall s\leq t
\end{equation}
In this work, we consider three specific HiPPO parameterizations of $\A$ and $\B$: LegT and LegS based on Legendre polynomials and FouT, which is based on a Fourier basis. The explicit constructions for $\A$ and $\B$ are given in Appendix \ref{parametrizations}.

%% file: sections/solving_dynamical_systems_with_ssm.tex
\section{Solving Dynamical Systems with SSMs}
\label{sec:solving_dynamical_syemstems_with_SSMs}
As alluded to in the introduction, we study SSMs on autoregressive ICL, predicting the next value $f_{k+1}$ given $f_{\leq k}$ in context. For SSMs, this corresponds to predicting $u_{k+1}$ after iteratively observing the first $k$ values of the sequence $u_{\leq k}$. To progress towards this goal, we consider a continuous relaxation of the problem where we map the indices to instances in continuous time, i.e., $k \to t, k+1 \to t + \Delta t$. In the following, we denote with $u(t)$ the input in continuous time and with $u_{k}$ its discrete counterpart. Our aim in the continuous setting is, therefore, to predict $u(t + \Delta t)$, for which we consider the following integral expression:
\begin{equation}\label{eq: disc}
    u(t+\Delta t) = u(t) +\int_{t}^{t+\Delta t}\dot{u}(s) \mathrm{d}s
\end{equation}
To solve this, we take multiple steps: \textbf{(1)} we approximate $\dot{u}(t)$ by constructing a specific parameterization for $\C$ and $\D$ for continuous SSMs and a general basis $\{p_n\}_{n \geq 0}$ resulting in a model satisfying $\dot{u}(t) \approx y(t)=\C x(t)+\D u(t)$. Subsequently, \textbf{(2)} we show how to approximate the integral via discretization, bringing us back to the discrete SSM setting and our original problem of autoregressive ICL. Finally, \textbf{(3)}, we provide an asymptotic bound on the error incurred by approximating $\dot{u}(t)$ with a finite hidden state.
\par
\textbf{(1)}~~ By evaluating Equation \ref{eq:reconstruction} at time $t$, we get: $u(t) = \sum_{n=0}^{\infty} x_n(t) p_n(t,t)$. Under some technical assumptions we can exchange the series with the derivative\footnote[2]{If $f(x)=\sum_{n=0}^\infty f_n(x)$, then if $\sum_{n=0}^\infty \dot{f}_n(x)$ converges absolutely we have that $\dot{f}(x)=\sum_{n=0}^\infty \dot{f}_n(x)$. We assume this assumption holds throughout the rest of the paper.}. Noting that $p_n(t,t)$ is a constant with respect to $t$, we get $\dot{u}(t) = \sum_{n=0}^{\infty} \dot{x}_n(t) p_n(t,t)$. Through this, we establish a weight construction in Proposition \ref{prop:construction_legt_fout}, such that the continuous SSM approximates the gradient of the input signal. The following proposition is for the LegT and FouT bases. In Appendix \ref{app:legs} we further provide the result for the LegS basis.
\begin{restatable}[Construction of $\C$ and $\D$ for LegT and FouT]{prop}{ConstrLegTFouT}
    \label{prop:construction_legt_fout}
    If we choose $\C_j = \sum_{k = 0}^N\A_{kj}p_k(t, t)$ and $
    \D = \sum_{k = 0}^N \B_k p_k(t, t)$ and $\A$, $\B$ and $p_k(t,t)$ as in HiPPO LegT or FouT, then the output $y(t) =: \dot{u}_N(t)$ is an approximation of $\dot{u}(t)$ based on $N$ basis functions.
\end{restatable}
\begin{proof}
We first assume an infinite hidden state size $N = \infty$, then use the definition of $\dot{x}(t)$ and truncate the series to obtain the result.
    \begin{align*}
        \dot{u}(t) &= \sum_{k=0}^\infty \dot{x}_k(t)p_k(t,t) = \sum_{k=0}^\infty\left(\sum_{j=0}^\infty \A_{kj}x_j(t) + \B_k u(t)\right)p_k(t,t)\\
        &\approx \sum_{j = 0}^N\left(\sum_{k = 0}^N \A_{kj}p_k(t, t)\right)x_j(t) + \left(\sum_{k = 0}^N \B_k p_k(t, t)\right)u(t)\\ &=: \sum_{j = 0}^N \C_j x_j(t) + \D u(t)
     \end{align*}
\end{proof}
\textbf{(2)}~~ Since we cannot solve the integral in Equation \ref{eq: disc} in closed form, we approximate it using the bilinear method, which then brings us back to the discrete SSM setting and our original problem of autoregressive ICL.
Starting from Equation \ref{eq: disc}, we approximate $\dot{u}(t) \approx \C x(t)+\D u(t)$ and then apply the bilinear method. Here, $\C$ and $\D$ are the parametrizations of HiPPO LegT or FouT, as defined in Proposition \ref{prop:construction_legt_fout}.
\begin{align*}
    u(t+ \Delta t) - u(t) & \approx\int_t^{t+\Delta t} \C x(s)+\D u(s) \mathrm{d}s \\
    & \approx \frac{\Delta t}{2} \Big( \C x(t) + \D u(t) + \C x(t + \Delta t)  + \D u(t + \Delta t)\Big). 
    \end{align*}
Rearranging the equation, then applying the discrete-time mapping $t \to k, t+\Delta t \to k+1$ as previously described, and approximating $x_k \approx x_{k+1}$ for small $\Delta t$, we obtain:
\begin{align*} \label{CD}
    \widehat{u}_{k+1} &= \left( 1 - \frac{\D \Delta t}{2}\right)^{-1} \left[ \left( 1 + \frac{\D \Delta t}{2}\right)u_k + \frac{\Delta t}{2} \C  \left(x_k + x_{k+1}\right)\right]\\
    &\approx \left( 1 - \frac{\D \Delta t}{2}\right)^{-1} \left[ \left( 1 + \frac{\D \Delta t}{2}\right)u_k + \Delta t \C x_{k+1} \right]\\
    &= \overline{\C} x_{k+1} + \overline{\D} u_k
\end{align*}
This yields a discrete-time system, where the hidden state evolution is given by $x_{k+1} = \overline{\A} x_{k}+\overline{\B} u_k$. Here, $\overline{\A} = (I-\frac{\Delta t}{2}\A)^{-1}(I+\frac{\Delta t}{2}\A)$ and $\overline{\B} = \Delta t (I-\frac{\Delta t}{2}\A)^{-1}\B$ are the discretized versions of $\A$ and $\B$, respectively \citep{Gu2020Hippo}. The complete discretized system is expressed as:
\begin{equation}
\label{discretized_sys}
    \begin{cases}
    x_{k+1} = \overline{\A} x_{k}+\overline{\B} u_k \\
    \widehat{u}_{k+1}=\overline{\C} x_{k+1} + \overline{\D} u_k 
\end{cases}
\end{equation}
Crucially, the above system allows us to predict the value of the future input state $\widehat{u}_{k+1}$ based on the hidden state $x_{k}$ (which is a function of $u_{< k}$) and the input $u_k$, and hence perform autoregressive ICL. Unlike classical machine learning, which requires training for specific dynamical systems, our parametrization predicts future states of arbitrary sequences without task-specific fine-tuning.
\par
\textbf{(3)}~~ To further investigate the proposed parametrization for the continuous-time SSM, we provide an asymptotic bound on the error incurred when approximating $\dot{u}(t)$ with $\dot{u}_N(t)$ which is the output of the continuous SSM with a finite hidden state of dimension $N$. For this, we consider an alternative construction of the FouT basis in Proposition \ref{fout_altern} simplifying the analysis. The proof of Proposition \ref{fout_altern} is analogous to that of Proposition \ref{prop:construction_legt_fout} and can be found in Appendix \ref{App: Fout}.
\begin{restatable}[Alternative FouT Construction]{prop}{FouTAlt}
\label{fout_altern}
If we choose , $\C_k = \begin{cases} 0 & \text{if} \: k=0 \text{ or } k \text{ odd}\\ -2 \sqrt{2}\pi n & \text{otherwise} \quad \end{cases}$\\ $\D = 0$ and $\A$, $\B$ and $p_k(t,s)$ as in HiPPO FouT and if $u(t)$ has $k$-th bounded derivatives for some $k \geq 3$, then the output $y(t) =: \dot{u}_N(t)$ is an approximation of $\dot{u}(t)$ based on $N$ basis functions.
\end{restatable}
Using this, we show that the error $\left|\dot{u}(t) - \dot{u}_N(t)\right|$, decreases polynomially in the hidden state size $N$ and linearly depends on the Lipschitz constant $L$ of the $(k-1)$-th derivative.
\begin{restatable}[Approximation Error]{thm}{ApproxErr}
\label{local_approx_error}
    If $u$ has $k$-th bounded derivatives for some $k \geq 3$, i.e. $|u^{(k)}(t)| \leq L \ \forall t$ then it holds that for the choice of $\A$ and $\B$ in HiPPO FouT and $\C$ and $\D$ as in Prop. \ref{fout_altern}: $\left|\dot{u}(t) - \dot{u}_N(t)\right| \in \mathcal{O}(L/N^{k-2})$
\end{restatable}
The proof of this Theorem can be found in Appendix \ref{App: Fout}. From this result, we can derive a Corollary for the error of predicting $u(t)$ using $u_N(t) = \int_0^{t}\dot{u}_N(s)\mathrm{d}s$ in the continuous setting:
\begin{restatable}[Approximation Error]{cor}{ApproxErrCor}
 \label{cor:ApproxErrCor}
     Under the same assumptions and parametrization as in Theorem \ref{local_approx_error}, we have that:
      $\left|u(t) -u_N(t)\right| \in \mathcal{O}\left(Lt/N^{k-2}\right)$
\end{restatable}
We further note that having $u_N(t) = \int_0^t \dot{u}_N(s) \mathrm{d}s$ reflects how we calculate $u_{k+1}$ in practice, with the difference that here we do not consider an approximation of the integral. In particular, $u_N(t)$ can be seen as the equivalent continuous version of our estimator $\hat{u}_k$.

%% file: sections/experiments.tex
\section{Experiments}
\label{sec:experiments}
We perform a thorough experimental evaluation of the weight construction presented in Equation \ref{discretized_sys}. For this, we unroll the model step-by-step to predict $u_{k+1}$ given $u_{\leq k}$ and evaluate the performance using $\mathcal{L}(\theta) = \frac{1}{T - T_s} \sum_{k = T_s}^T \left|f_{\theta}(u_{k}, x_{k}) - u_{k+1} \right|$, where $\theta = \{\overline{\A}, \overline{\B}, \overline{\C}, \overline{\D}\}$ and $f_\theta$ is the parametrized SSM. Unless specified otherwise, we use \smash{$T = 10^4$} and $T_s=T/2$, providing the model with sufficiently long context.\par
\textbf{Ordinary Differential Equation}~~ 
 We compare our model's ability to predict the next state of an adapted  Van der Pol Oscillator in one variable: $\dot{u}(t) = \mu (1-u(t)^2) \sin(t)$, with $N=65$. As seen in \ref{subfig:van_der_pol}, both LegT and FouT achieve lower error in regions of lower curvature, consistent with the dependence on the Lipschitz constant established in Theorem \ref{local_approx_error}.\par
\textbf{White Signal and Filtered Noise Process}~~ Following \citet{Gu2020Hippo} we use band-limited and low pass filtered white noise \ref{subfig:example_signals}, which we refer to as White Signal a Filtered Noise respectively (see \cite{nengo}). For each, we use three progressively harder setups in which the high-frequency information is increased. The frequency content is controlled by $\alpha$ for Filtered Noise and by $\gamma$ for White Noise. Smaller $\alpha$ and larger $\gamma$ correspond to increased high-frequency content. We test the model over different hidden state sizes ranging from 1 to 96 in increments of 5. We find that using a larger hidden state dimension is beneficial up to a certain $N$, after which for LegT there is minimal to no further benefit. For FouT, performance initially worsens when increasing $N$ before improving. This is because, for low $N$, the model essentially copies its input, which is surpassed for larger $N$. Furthermore, the error is lower with less high-frequency content.
\par
\textbf{Learning from Context}~~ To empirically demonstrate our models' use of context, we examine how the error scales with increased context length. In \ref{subfig:time_dependence}, we find that for all hidden state sizes $N$, the model gets better with a longer context. Furthermore, more expressive models with larger hidden states exhibit longer oscillations, requiring more samples to stabilize. Intuitively, larger $N$ corresponds to bigger function classes that the model can represent.

\input{sections/main_figure}

\textbf{Constructions as Initialisation}~~ \looseness=-1 In \ref{subfig:construction_vs_training}, we compare the performance of SSM layers in different settings. $(\mathrm{I})$ Initializing $\overline{\A}$, $\overline{\B}$, $\overline{\C}$, $\overline{\D}$ at construction and training $\overline{\C},\overline{\D}$. $(\mathrm{II})$ Fixing $\overline{\A}$, $\overline{\B}$, $\overline{\C}$, $\overline{\D}$ at construction. $(\mathrm{III})$ Fixing $\overline{\A},\overline{\B}$ at construction, standard Gaussian initialization and training of $\overline{\C},\overline{\D}$. $(\mathrm{IV})$ Initializing $\overline{\A},\overline{\B},\overline{\C},\overline{\D}$ at construction and training all of them. We train the models on a mixed dataset, consisting of White Signal, Legendre Polynomials, and sums of sine functions (see Appendix \ref{subsec:experiment3} for further details). We evaluate these models on 3 hold-out datasets, namely Filtered Noise, a holdout mixed dataset, linear functions, and the Van der Pol oscillator from \ref{subfig:van_der_pol}. We observe that initializing the SSM with our parametrization $(\mathrm{I}, \mathrm{II}, \mathrm{IV})$ leads to enhanced predictive performance over random initialization $(\mathrm{III})$. More so, training the model using gradient methods $(\mathrm{I}, \mathrm{III}, \mathrm{IV})$ does not result in increased performance over our weight construction $(\mathrm{II})$.\looseness=-1

%% file: sections/main_figure.tex
\begin{figure*}[t]
\centering
\subfigure[\small FouT and LegT on Van der Pol]{
    \includegraphics[width=0.31\textwidth]{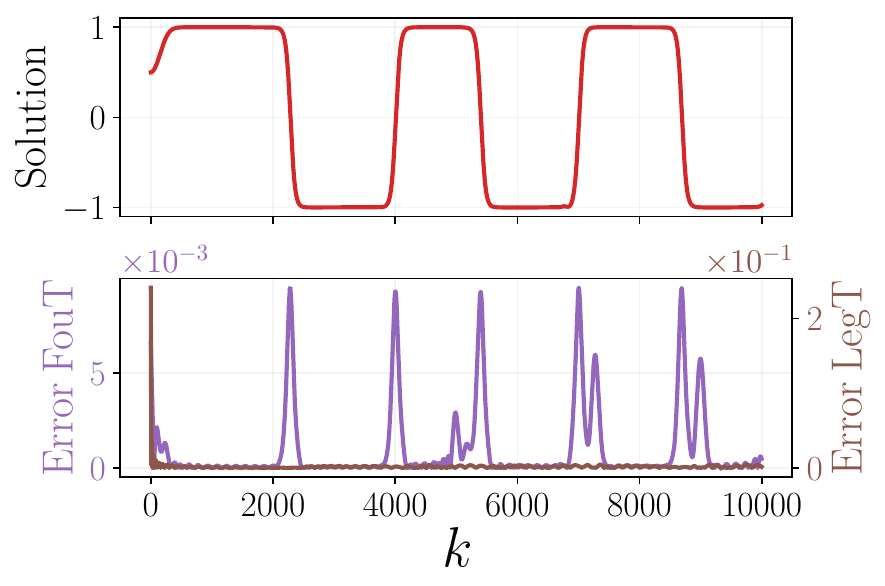}
    \label{subfig:van_der_pol}
}
\hfill
\subfigure[\small Error dependence on $N, \alpha$]{
    \includegraphics[width=0.31\textwidth]{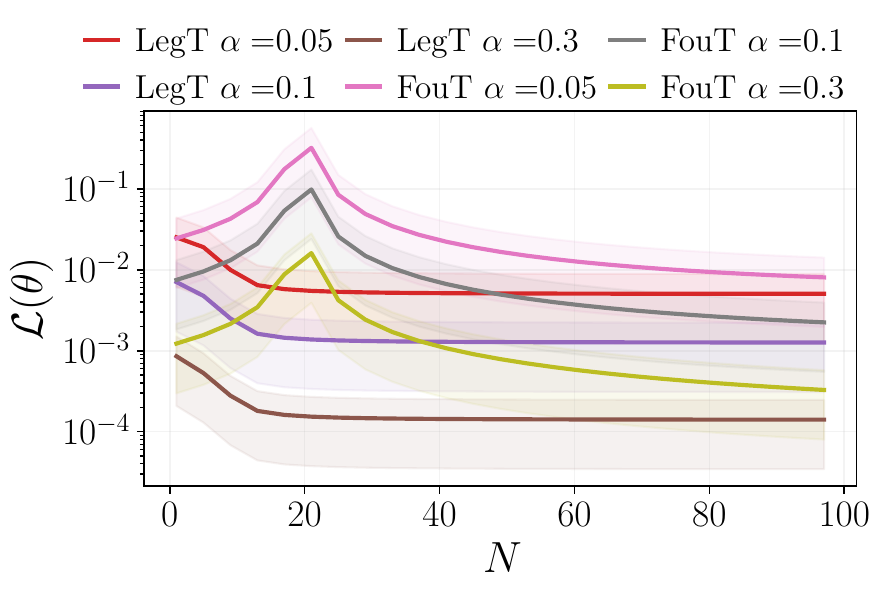}
    \label{subfig:filtered_error}
}
\hfill
\subfigure[\small Error dependence on $N, \gamma$]{
    \includegraphics[width=0.31\textwidth]{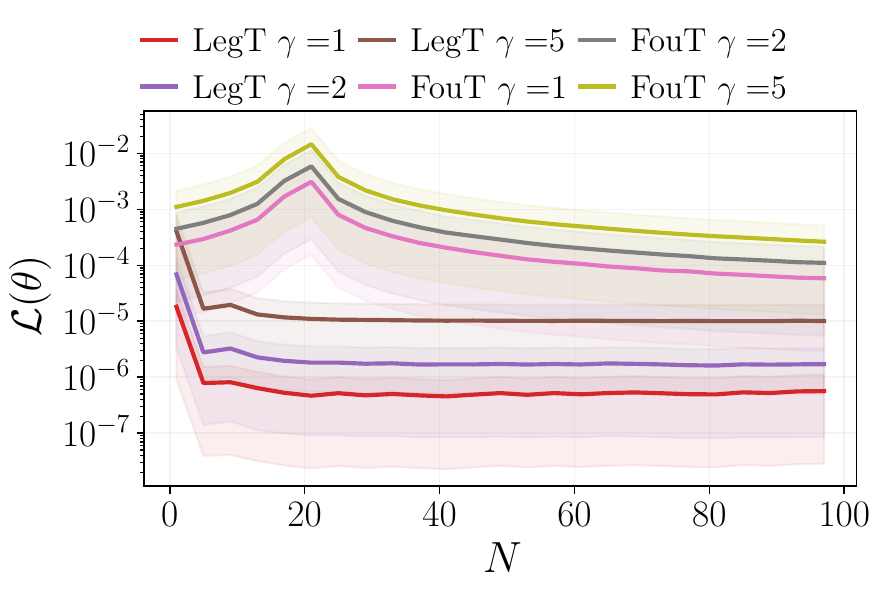}
    \label{subfig:whitenoise_error}
}

\vfill

\subfigure[\small Construction vs. trained SSM]{
    \includegraphics[width=0.31\textwidth]{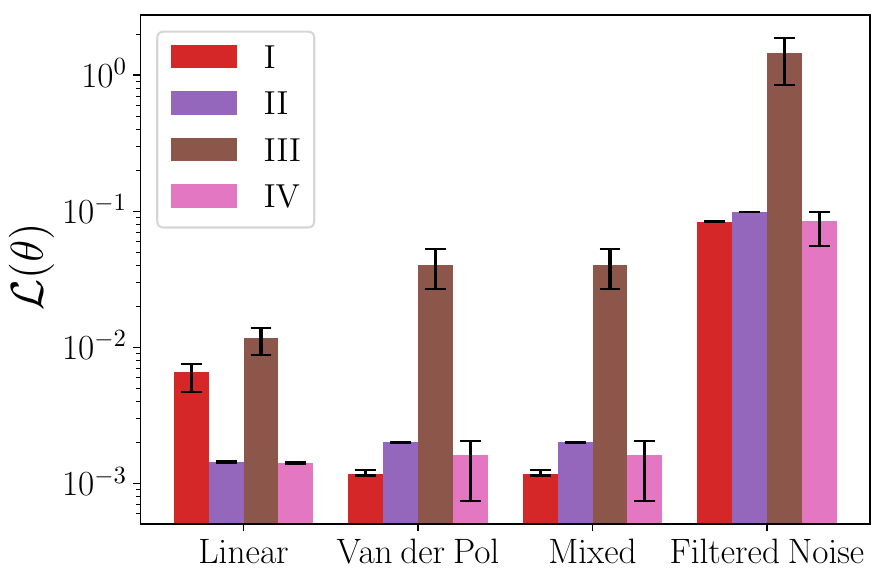}
    \label{subfig:construction_vs_training}
}
\hfill
\subfigure[\small LegT error dependence on $k$]{
    \includegraphics[width=0.32\textwidth]{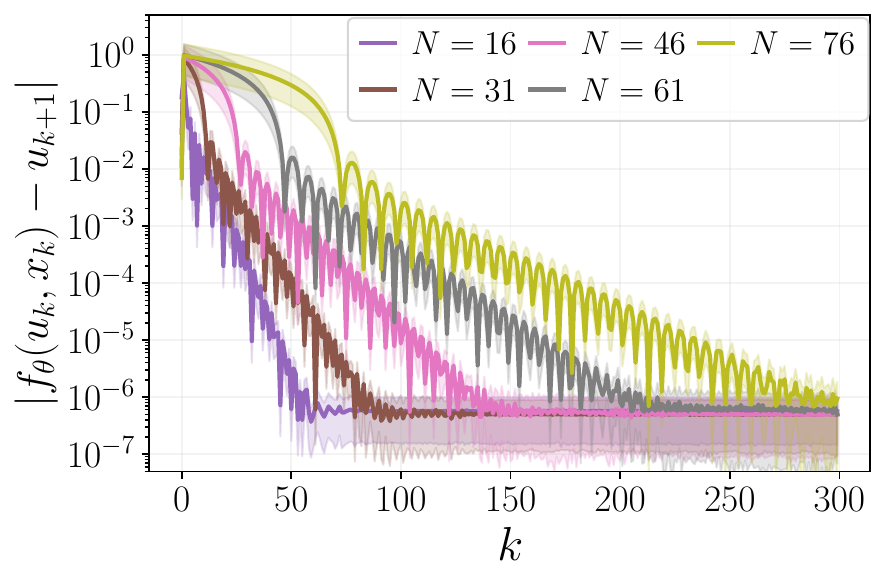}
    \label{subfig:time_dependence}
}
\hfill
\subfigure[\small Filtered Noise \& White Signal]{
    \includegraphics[width=0.31\textwidth]{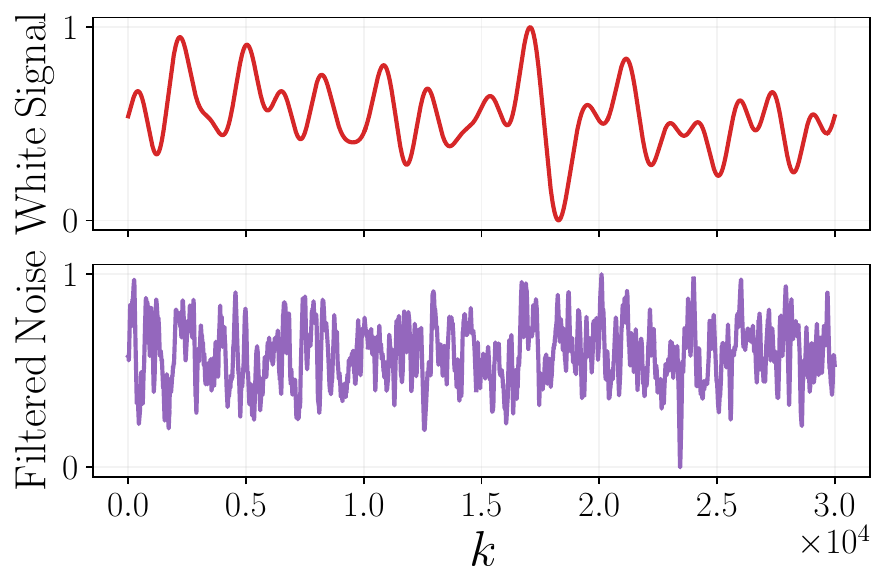}
    \label{subfig:example_signals}
}

\caption[Short Caption]{Empirical evaluation of our weight construction. \subfigref{subfig:van_der_pol} Van der Pol oscillator solution and errors $|f_\theta(u_k,x_k) - u_{k+1}|$ of FouT and LegT weight construction. \subfigref{subfig:filtered_error} Error dependence on $N$ and $\alpha$ for the Filtered Noise dataset (mean across 1k functions and 1 std. plotted). \subfigref{subfig:whitenoise_error} Error dependence on $N$ and $\gamma$ for the White Signal dataset (mean across 1k functions and 1 std. plotted). \subfigref{subfig:construction_vs_training} Performance comparison of weight construction for LegT: $(\mathrm{I})$ Initializing $\overline{\A},\overline{\B},\overline{\C},\overline{\D}$ at construction and training $\overline{\C},\overline{\D}$. $(\mathrm{II})$ Fixing $\overline{\A},\overline{\B},\overline{\C},\overline{\D}$ at construction. $(\mathrm{III})$ Fixing $\overline{\A},\overline{\B}$ at construction, standard Gaussian initialization and training of $\overline{\C},\overline{\D}$. $(\mathrm{IV})$ Initializing $\overline{\A},\overline{\B},\overline{\C},\overline{\D}$ at construction and training all of them (mean over 3 random seeds and error bars correspond to min and max). \subfigref{subfig:time_dependence} Error dependence of LegT on the context length $k$ on the White Signal Dataset (mean across 1k functions and 1 std. plotted). \subfigref{subfig:example_signals} Examples of the Filtered Noise ($\alpha = 0.05$) and White Signal ($\gamma = 2$) datasets}
\label{fig:label}
\end{figure*}

%% file: sections/conclusions.tex
\section{Conclusions}
In this work, we propose a novel SSM construction that can predict the next state of any dynamical system from its history without fine-tuning its parameters. To the best of our knowledge, this is the first time it was theoretically shown that SSMs can perform autoregressive ICL. We find that our weight construction allows SSMs to effectively leverage context to make predictions and that it can serve as a good initialization.\par 
This work serves as an initial step towards understanding the ICL capabilities of GSSMs and opens several avenues for future research. Investigating how gating mechanisms in modern GSSMs like Mamba \cite{gu2023mamba} and Griffin \cite{de2024griffin} affect ICL capabilities is one potential direction. Another is examining the impact of fully connected layers and non-linearities following SSM blocks. Lastly, we observed instabilities when predicting multiple steps into the future. Exploring methods to improve multi-step predictions and understanding the instabilities could also be valuable.

%% file: sections/appendix.tex
\section{Parametrizations for LegT, LegS and FOuT}
\label{parametrizations}
\subsection{HiPPo-LegT}
As mentioned in the main text, HiPPO-LegT uses Legendre polynomials as its basis functions. In particular, we have that $\omega(t,s)= \frac{1}{\theta}\mathbb{I}_{[t-\theta, t]}(s)$ and $p_n(t,s)=\sqrt{2n+1}P_n(2\frac{s-t}{\theta}+1)$ where $P_n$ is the $n$-th Legendre polynomial. This leads to the following choice for $\A$ and $\B$:
\begin{align}
    \A_{nk} = \frac{1}{\theta} \begin{cases}
        (-1)^{n-k}(2n+1) \ \ \ \text{if $n \geq k$}\\
        (2n+1) \ \ \ \text{if $n < k$}
    \end{cases} 
    \ \ \ \ \ \ 
    \B_n=\frac{1}{\theta}(2n+1)(-1)^n
\end{align}
\subsection{HiPPo-LegS}
 HiPPO-LegS also uses Legendre polynomials as its basis functions. However, here we have that $\omega(t,s)= \frac{1}{t}\mathbb{I}_{[0, t]}(s)$ and $p_n(t,s)=\sqrt{2n+1}P_n(\frac{2s}{t}-1)$ where $P_n$ again is the $n-$th Legendre polynomial. Note that whereas the measure of LegT is translation invariant, the measure of LegS is not. This causes the system to become Time-Varying and in particular this leads to having $\dot{x}(t)=-\frac{1}{t}\A x(t)+\frac{1}{t}\B u(t)$ with the following choice for $\A$ and $\B$:
\begin{align}
    \A_{nk} =  \begin{cases}
        \sqrt{2n+1}\sqrt{2k+1} \ \ \ \text{if $n > k$}\\
        (n+1) \ \ \ \text{if $n = k$}\\
        0 \ \ \ \text{o.w.}
    \end{cases} 
    \ \ \ \ \ \ 
    \B_n=\sqrt{2n+1}
\end{align}
\subsection{HiPPO-FouT}
HiPPO-FouT, differently from LegT and LegS uses the classical Fourier basis and in particular it assumes $\{p_n\}_{n \geq 0}(t,s) = \sqrt{2} \Big [1 \ \cos(2\pi [1-(t-s)]) \ \sin(2\pi [1-(t-s)]) \ \cos(4\pi [1-(t-s)]) \ \sin(4\pi [1-(t-s)]) \ \dots \Big]$ and $\omega(t,s)= \mathbb{I}_{[t-1, t]}(s)$. This leads to the following choice of $\A$ and $\B$:
\begin{align}
    \A_{nk} = 
    \begin{cases}
        -2 \ \ \ \text{if $n=k=0$}\\
        -2\sqrt{2} \ \ \ \text{if $n=0$ and $k$ odd or $k=0$ and $n$ odd}\\
        -4 \ \ \ \text{if $n$ odd and $k$ odd}\\
        2\pi n \ \ \ \text{if $n-k=1$ and $k$ odd}\\
        -2\pi k \ \ \ \text{if $k-n=1$ and $n$ odd}\\
        0 \ \ \ \text{o.w.}
    \end{cases}
    \ \ \ \ 
    \B_n= 
    \begin{cases}
        2 \ \ \ \text{if $n=0$}\\
        2 \sqrt{2} \ \ \ \text{if $n$ odd}\\
        0 \ \ \ \text{o.w.}
    \end{cases}
\end{align}
\section{Proofs}
\label{proofs}
\subsection{Construction of \texorpdfstring{$\C$}{} and \texorpdfstring{$\D$}{} for LegT and FouT}
\label{app:legt_fout_constr}
\ConstrLegTFouT*
\begin{proof} We first assume that the hidden state has infinite dimension (i.e. $N = \infty$), such that we have perfect reconstruction of $\dot{u}(t)$:
\begin{align*}
    \dot{u}(t) &= \sum_{k = 0}^\infty \dot{x}_k(t) p_k(t, t)\\
    &= \sum_{k = 0}^\infty\left( \sum_{j = 0}^\infty \A_{kj}x_j(t) + \B_k u(t)\right) p_k(t, t) \ \ \ \  \ \text{by the definition of $\dot{x}(t)$}\\
    &\approx \sum_{k = 0}^N\left( \sum_{j = 0}^N \A_{kj}x_j(t) + \B_k u(t)\right) p_k(t, t)\ \ \ \  \ \text{approximating to finite hidden dimension} \\
    &= \sum_{j = 0}^N\left(\sum_{k = 0}^N \A_{kj}p_k(t, t)\right)x_j(t) + \left(\sum_{k = 0}^N \B_k p_k(t, t)\right)u(t) \\
    &= \sum_{j = 0}^N \C_j x_j(t) + \D u(t)
\end{align*}
Therefore, we have:
\begin{align}
    \C_j = \sum_{k = 0}^N\A_{kj}p_k(t, t)\\
    \D = \sum_{k = 0}^N \B_k p_k(t, t)
\end{align}
\end{proof}
\subsection{Construction of \texorpdfstring{$\C$}{} and \texorpdfstring{$\D$}{} for LegS}
\label{app:legs}
\begin{restatable}[LegS construction]{prop}{LegS}
    \label{prop:LegS}
     If we choose $\C_j = \sum_{k = 0}^N\frac{1}{t}\A_{kj}p_k(t, t)$ and $
    \D = \sum_{k = 0}^N \frac{1}{t}\B_k p_k(t, t)$ and $\A$, $\B$ and $p_k(t,t)$ as in HiPPO LegS, then the output $y(t)$ is an approximation of $\dot{u}(t)$ based on $N$ basis functions.
 \end{restatable}
\begin{proof}
Now $\A$ and $\B$ represent the parametrizations of HiPPO LegS. The proof exactly follows the same steps as the one above, by first assuming infinite hidden dimension and then approximating to finite dimension $N$.
\begin{align*}
    \dot{u}(t) &= \sum_{k = 0}^\infty \dot{x}_k(t) p_k(t, t)\\
    &= \sum_{k = 0}^\infty\left( \sum_{j = 0}^\infty \frac{1}{t}\A_{kj}x_j(t) + \frac{1}{t}\B_k u(t)\right) p_k(t, t)\\
    &\approx \sum_{k = 0}^N\left( \sum_{j = 0}^N \frac{1}{t}\A_{kj}x_j(t) + \frac{1}{t}\B_k u(t)\right) p_k(t, t)\\
    &= \sum_{j = 0}^N\left(\sum_{k = 0}^N \frac{1}{t}\A_{kj}p_k(t, t)\right)x_j(t) + \left(\sum_{k = 0}^N \frac{1}{t}\B_k p_k(t, t)\right)u(t)\\
    &= \sum_{j = 0}^N \C_j x_j(t) + \D u(t)
\end{align*}
Therefore, we have:
\begin{align}
    \C_j = \sum_{k = 0}^N \frac{1}{t}\A_{kj}p_k(t, t)\\
    \D = \sum_{k = 0}^N \frac{1}{t}\B_k p_k(t, t)
\end{align}
\end{proof}
\section{Discretization}
\subsection{Discretization of \texorpdfstring{$\C$}{} and \texorpdfstring{$\D$}{}}
\label{app:discretization}
The trapezoid (or bilinear) method is one of the most widely used methods in numerical analysis, used in general to approximate integrals by:
\begin{equation}
    \int_a^b f(s) \mathrm{d}s \approx \frac{b-a}{2}\big(f(a)+f(b)\big)
\end{equation}
In practice, we are approximating the integral between $a$ and $b$ with the area of the trapezoid with height $b-a$ and parallel sides $f(a)$ and $f(b)$, which leads to a good approximation if $b-a$ is small.\\
In our setting, we have that:
\begin{align}
    u(t+ \Delta t) - u(t) &=\int_t^{t+\Delta t} \dot{u}(s) \mathrm{d}s \\
    &= \int_t^{t+\Delta t} \C x(s)+\D u(s) \mathrm{d}s \ \ \ \ \text{by using our construction for $\dot{u}(s)$}\\
    & \approx \frac{\Delta t}{2} \Big( \C x(t) + \D u(t) + \C x(t + \Delta t)  + \D u(t + \Delta t)\Big) 
\end{align}
where in the last step we used the trapezoid rule presented above. Then, by rearranging the terms, we get:
\begin{align}
    u(t + \Delta t) &= \left( 1 - \frac{\D \Delta t}{2}\right)^{-1} \left[ \left( 1 + \frac{\D \Delta t}{2}\right)u(t) + \frac{\Delta t}{2} \C  \left(x(t) + x(t + \Delta t)\right)\right]\\
    &\approx \left( 1 - \frac{\D \Delta t}{2}\right)^{-1} \left[ \left( 1 + \frac{\D \Delta t}{2}\right)u(t) + \Delta t \C x(t+\Delta t) \right]\\
    &= \overline{\C} x(t+\Delta t)+ \overline{\D} u(t)
\end{align}
where we approximated $x(t) \approx x(t+\Delta t)$ and where $\overline{\C} = \Delta t \left( 1 - \D \Delta t/2\right)^{-1} \C$ and $\overline{\D} = \left( 1 - \D \Delta t/2\right)^{-1}\left( 1 + \D \Delta t/2\right)$.\\
We then map the discretized timestep to indices, e.g. if $t \rightarrow k$ then we map $t+\Delta t \rightarrow k+1$, which leads to:
\begin{equation*}
    u_{k+1} = \overline{\C} x_{k+1} + \overline{\D} u_k
\end{equation*}
Note that this discretization scheme works for all HiPPO LegT, LegS and FouT.
\subsection{Discretization of Continuous Linear Systems}
\label{linear_discretization}
Given a continuous time dynamical linear system of the form:
\begin{equation*}
    \begin{cases}
        \dot{x}(t) = \A x(t)+\B u(t) \\
        y(t) = \C x(t)+\D u(t)
    \end{cases}
\end{equation*}
In order to be able to use this system in practice and experiments we need to discretize it. There are many methods used to this end, like Forward Euler, Bilinear or FoH. By employing them, we arrive at a discrete time dynamical system of the form:
\begin{equation}
\label{eq:discr_sys}
    \begin{cases}
        x_{k+1} = \overline{\A} x_k+\overline{\B} u_k \\
        y_k = \overline{\C} x_{k+1}+\overline{\D} u_k
    \end{cases}
\end{equation}
Where, in the case of the bilinear method, which we use in this work, we have $\overline{\A} = (I-\frac{\Delta t}{2}\A)^{-1}(I+\frac{\Delta t}{2}\A)$, $\overline{\B} = \Delta t (I-\frac{\Delta t}{2}\A)^{-1} \B$, $\overline{\C} = (I-\frac{\Delta t}{2}\A)^{-\top} \C^\top$ and $\overline{\D} = \D +\frac{1}{2} \C^\top \overline{\B}$.
\par
Note that in our setting this would imply discretizing $\C$ and $\D$ twice, once for approximating $\dot{u}(t)$ to predict the next value of $u(t)$ (see Appendix \ref{app:discretization}) and once for discretizing the full system as in Equation \ref{eq:discr_sys}. We find that only employing the first one works better in practice then using both of them. Hence, in our work, we discretize $\A$ and $\B$ according to the bilinear method and $\C$ and $\D$ only according to the discretization shown in Appendix \ref{app:discretization}.
\section{Alternative FouT construction}
\begin{lemma}
\label{lemma:bound_x}
    If $u(t)$ has $k$-th bounded derivatives, i.e. $|u^{(k)}(t)| \leq L \ \forall t$, then we have that $|x_n^s(t)| \leq \frac{L}{(2\pi n)^k}$ and $|x_n^c(t)| \leq \frac{L}{(2\pi n)^k}$.
\end{lemma}
\begin{proof}
From Theorem 7 in \citep{gu2022train}, we have that if $u$ has $k$-th bounded derivatives, it holds that:
\begin{align*}
    |x_n^s(t)|  &\leq \left|\frac{1}{(2\pi n)^k}\int_0^1 u^{(k)}(t)p_{n}(s-t)\mathrm{d}t\right|\\
    &\leq \frac{1}{(2\pi n)^k}\int_0^1 \big|u^{(k)}(t)\big| \big|p_{n}(s-t)\big|\mathrm{d}t\\
    &\leq \frac{L}{(2\pi n)^k}
\end{align*}
Similarly, under the same assumptions, it holds that:
\begin{align*}
    |x_n^c(t)| & \leq \left|\frac{1}{(2\pi n)^k}\int_0^1 u^{(k)}(t)p_{n}(s-t)\mathrm{d}t\right|\\
    &\leq \frac{1}{(2\pi n)^k}\int_0^1 \big|u^{(k)}(t)\big| \big|p_{n}(s-t)\big|\mathrm{d}t\\
    &\leq \frac{L}{(2\pi n)^k}
\end{align*}
\end{proof}
\label{App: Fout}
\FouTAlt*
\begin{proof} As always, we start from:
\begin{equation*}
    u(s) = \sum_{n=0}^\infty x_n(t) p_n(t,s) \: \: \forall s \leq t
\end{equation*}
where in the case of Fourier basis, we denote: 
\begin{align*}
    \{p_n\}_{n\geq 0} &= \sqrt{2}\begin{bmatrix}
    1 & \cos(2\pi t) & \sin(2\pi t) & \cos(4\pi t) & \sin(4\pi t) & \dots
\end{bmatrix}^\top\\
x(t) &= \begin{bmatrix}
    x_0(t) & x_1^c(t) & x_1^s(t) & x_2^c(t) & x_2^s(t) & \dots
\end{bmatrix}^\top \in \mathbb{R}^{2N+1}
\end{align*}
Hence, we have that for all $s \leq t$:
\begin{equation*}
    u(s) = \sqrt{2}x_0(t) + \sqrt{2}\sum_{n=1}^\infty x_n^c(t) \cos\left(2\pi n \big[(t-s) + 1\big]\right) + \sqrt{2}\sum_{n=1}^\infty x_n^s(t) \sin\left(2\pi n \big[(t-s) + 1\big]\right)
\end{equation*}
Assuming $u$ has k-th bounded derivative for some $k \geq 3$, it follows from Lemma \ref{lemma:bound_x} that both \\
$\sum_{n=1}^\infty n x_n^c(t) \sin\left(2\pi n \big[(t-s)+1\big]\right)$ and $\sum_{n=1}^\infty n x_n^s(t) \cos\left(2\pi n \big[(t-s) + 1\big]\right)$ absolutely converge and hence we can exchange the series with the derivative and then we have for all $s < t$:
\begin{equation*}
    \dot{u}(s) = 2\sqrt{2}\pi\sum_{n=1}^\infty n x_n^c(t) \sin\left(2\pi n \big[(t-s) + 1\big]\right) - 2\sqrt{2}\pi\sum_{n=1}^\infty n x_n^s(t) \cos\left(2\pi n \big[(t-s) + 1\big]\right)
\end{equation*}
By assuming continuity of $\dot{u}$ (which holds since $u$ is $k$-times differentiable with $k \geq 3$), we can take the limit as $s$ goes to $t$ to extend the derivative and get:
\begin{equation*}
    \dot{u}(t) = -2\sqrt{2}\pi \sum_{n=1}^\infty n x_n^s(t) \approx \sum_{k=0}^{N}\C_k x_k(t)
\end{equation*}
Thus, we have that:
\begin{align*}
    \C_k &= \begin{cases}
        0 & \text{if} \: k=0 \text{ or } k \text{ odd}\\
        -2\sqrt{2}\pi k & \text{otherwise} \quad
    \end{cases}\\
    \D &= 0
\end{align*}
\end{proof}
\subsection{Approximation Error}
\ApproxErr*
\begin{proof} 
Let's recall the definition of $\dot{u}_N$:
\begin{equation*}
\dot{u}_N(t) := -\sqrt{2} \sum_{n=1}^{N-1} 2\pi n x_n^s(t)
\end{equation*}
Using Lemma \ref{lemma:bound_x}, we can now proceed to our bound. Assuming $k \geq 3$ to guarantee convergence of the infinite series, we get:
\begin{align*}
     \left|\dot{u}(t) - \dot{u}_N(t)\right| &= \left|-\sqrt{2} \sum_{n=N}^{\infty} 2\pi n x_n^s(t)\right|\\
     &\leq 2\sqrt{2}\pi \sum_{n=N}^{\infty} n \left|x_n^s(t)\right|\\
     &\leq 2\sqrt{2}\pi \sum_{n=N}^{\infty}n\frac{L}{(2\pi n)^k}\\
     &\in \mathcal{O}\left(\frac{L}{N^{k-2}}\right)
\end{align*}
\end{proof}
\ApproxErrCor*
\begin{proof}
This bound simply follows from Theorem \ref{local_approx_error} and the fact that $u_N(t) = \int_0^t \dot{u}_N(s) ds$:
\begin{align*}
     \left|u(t) - u_N(t)\right| &= \left| \int_0^t \dot{u}(t) - \dot{u}_N(t)\mathrm{d}t\right|\\
     &\leq \int_0^t \left|\dot{u}(t) - \dot{u}_N(t)\right|\mathrm{d}t\\
     &\in \mathcal{O}\left(\frac{Lt}{N^{k-2}}\right)
\end{align*}
Note that having $u_N(t) = \int_0^t \dot{u}_N(s) \mathrm{d}s$ reflects how we calculate $u(t+\Delta t)$ in practice, with the difference that here we do not consider an approximation of the integral.
\end{proof}
\section{Additional Experimental Details}
\label{App:exp}
\subsection{Additional Experiments}
\label{app:nengo}
\subsubsection{White Signal and Filtered Noise}
The White Signal and Filtered Noise datasets were generated via Nengo's \cite{nengo} White Signal and Filtered Noise functions. To further investigate the performance of FouT and LegT, we sample 100 different functions at random and evaluate both weight constructions. We report mean and standard deviation of the MSE for the two constructions in Table \ref{tab:wn_fn_table}. We compare state dimension $N=33$ and $N=65$ with 10'000 time steps per signal. For White Signal, we compare cut-off frequencies $\gamma = 0.3$, $1$, and $2$. For Filtered Noise, we use the Alpha filter \cite{nengo} with parameters $\alpha = 0.05$, $0.1$, and $0.3$. As $\gamma$ increases or $\alpha$ decreases, the resulting functions become more oscillatory and challenging to approximate (See Figure \ref{fig:fun_approx}). Note that we have shifted the ground truth function up by 0.1 to distinguish it from the model's predictions.  Functions from the White Signal process are smoother and easier to approximate, while Filtered Noise generates rougher, discontinuous functions. Unsurprisingly, our models perform better on White Signal data.
\begin{figure}[t]
    \includegraphics[width=0.33\linewidth]{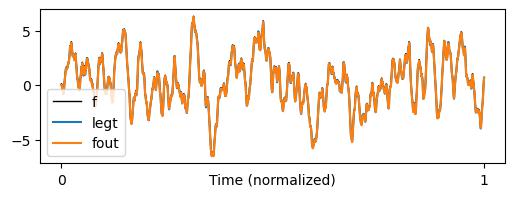}\hfill
    \includegraphics[width=0.33\linewidth]{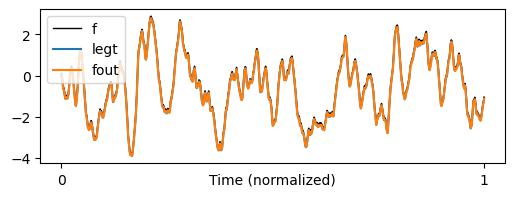}\hfill
     \includegraphics[width=0.33\linewidth]{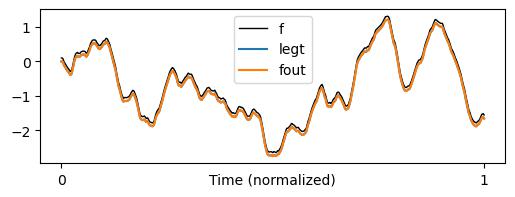}\hfill
     \includegraphics[width=0.33\linewidth]{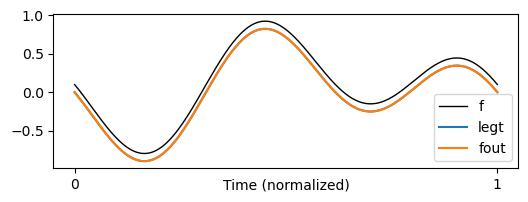}\hfill
     \includegraphics[width=0.33\linewidth]{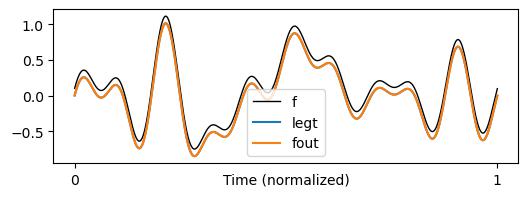}\hfill
    \includegraphics[width=0.33\linewidth]{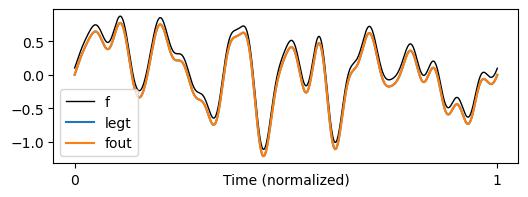}
     \caption{(Up) Functions sampled from Filtered Noise with $\alpha=0.05$, $0.1$ and $0.3$ (from left to right) and predictions from LegT and FouT. (Bottom) Functions sampled from White Signal with $\gamma=0.3$, $1$ and $2$ (from left to right) and predictions from LegT and FouT. In both cases we use $N=65$}
     \label{fig:fun_approx}
 \end{figure}
\begin{table}[t]
    \centering
    \scriptsize
    \begin{tabular}{c|c|c|c|c}
    & LegT ($N = 33$) & FouT ($N = 33$) & LegT ($N = 65$) & FouT ($N = 65$) \\ \hline
    White Signal ($\gamma = 0.3$) & $3.5 \pm 4.2$ $(1\mathrm{e}{-}11)$ & $6.8 \pm 5.8$ $(1\mathrm{e}{-}8)$ & $\mathbf{1.2 \pm 1.4}$ $\mathbf{(1\mathrm{e}{-}11)}$ & $6.9 \pm 5.8$ $(1\mathrm{e}{-}8)$ \\
    White Signal ($\gamma=1$) & $2.9 \pm 2.8$ $(1\mathrm{e}{-}7)$ & $2.1 \pm 1.2$ $(1\mathrm{e}{-}6)$ & $\mathbf{2.0 \pm 2.5}$ $\mathbf{(1\mathrm{e}{-}10)}$ & $2.1 \pm 1.2$ $(1\mathrm{e}{-}6)$ \\
    White Signal ($\gamma=2$) & $1.2 \pm 0.6$ $(1\mathrm{e}{-}5)$ & $8.6 \pm 3.7$ $(1\mathrm{e}{-}6)$ & $\mathbf{6.3 \pm 5.4}$ $\mathbf{(1\mathrm{e}{-}7)}$ & $8.7 \pm 3.8$ $(1\mathrm{e}{-}6)$ \\
    Filtered Noise ($\alpha=0.05$) & $2.1 \pm 0.2$ $(1\mathrm{e}{-}3)$ & $1.7 \pm 0.1$ $(1\mathrm{e}{-}3)$ & $2.8 \pm 0.3$ $(1\mathrm{e}{-}3)$ & $\mathbf{1.5 \pm 0.1}$ $\mathbf{(1\mathrm{e}{-}3)}$ \\
    Filtered Noise ($\alpha=0.1$) & $2.4 \pm 0.3$ $(1\mathrm{e}{-}4)$ & $1.9 \pm 0.2$ $(1\mathrm{e}{-}4)$ & $2.6 \pm 0.3$ $(1\mathrm{e}{-}4)$ & $\mathbf{1.8 \pm 0.2}$ $\mathbf{(1\mathrm{e}{-}4)}$ \\
    Filtered Noise ($\alpha=0.3$) & $5.0 \pm 1.0$ $(1\mathrm{e}{-}6)$ & $6.4 \pm 1.5$ $(1\mathrm{e}{-}6)$ & $\mathbf{4.1 \pm 0.6}$ $\mathbf{(1\mathrm{e}{-}6)}$ & $6.2 \pm 1.5$ $(1\mathrm{e}{-}6)$ \\ 
    \end{tabular}
    \caption{MSE of LegT and FouT on White Signal and Filtered Noise with $N=33$ and $N=65$.}
    \label{tab:wn_fn_table}
\end{table}
 \subsubsection{Approximating Differential Equations from Physics}
In this section, we provide additional results on the performance of our SSM construction to predict the next value of a dynamical system governed by an ODE. For this, we again consider the modified Van der Pol oscillator as described in the main text and provide additional results for the Bernoulli equation. The Bernoulli equation takes the following form:
\begin{equation}
   \dot{u}(t) + P(t)u(t)= Q(t) u(t)^n  
\end{equation}
We use $P(t) = \cos(5t)$, $Q(t)=\sin(t)$ and $n=\frac{1}{2}$. We plot the true solution and the prediction of LegT and FouT in Figure \ref{fig:physics_ode} (where again we shift the true function by 0.1 to avoid overlapping) and report the performance of the two methods in Table \ref{tab:physic_ode}. Again, we compare performance for state dimension $N=33$ and $N=65$. We notice that LegT significantly outperforms FouT by at least one order of magnitude for both $N=33$ and $N=65$ .\par
We now recall the equation for the Van der Pol's oscillator that we used in the main text:
\begin{equation}
     \dot{u}(t) = \mu (1-u(t)^2) \sin(t)
\end{equation}
where $\mu$ is a hyperparameter. In our experiments, we use $\mu=7$ and $N=17$. We plot the solution and our predictions in Figure \ref{fig:physics_ode} and Table \ref{tab:physic_ode}. We notice that for $N=33$, Legt and FouT perform comparably whereas for $N=65$ LegT outperforms FouT. We remark that by looking at Figure \ref{fig:physics_ode}, we can see that the solution for the Van der Pol's Oscillator is very similar to a square wave. Note that for this function it is very hard to predict the next value due to the very steep sudden increase/decrease in its value following a flat region. Hence, it does not come as a surprise to see that the performance for both LegT and FouT is much higher (by at least one order of magnitude across all the different values of $N$ that we tested) on the Bernoulli's equation, which is smoother. 
\begin{table}[t]
    \centering
    \small
    \begin{tabular}{c|c|c}
 & Bernoulli & Van der Pol\\
 \hline
 LegT ($N=33$) & $1.8$ $(1\mathrm{e}{-}8)$  &  $6.4 $ $(1\mathrm{e}{-}6)$\\
 FouT ($N=33$)& $3.0$ $(1\mathrm{e}{-}7)$ & $6.6$ $(1\mathrm{e}{-}6)$\\
 LegT ($N=65$) & $\mathbf{1.7}$ $(1\mathrm{e}{-}10)$  &  $\mathbf{4.4}$ $(1\mathrm{e}{-}8)$\\
 FouT ($N=65$)& $3.0$ $(1\mathrm{e}{-}7)$ & $6.6$ $(1\mathrm{e}{-}6)$\\
\end{tabular}
    \caption{MSE of LegT and FouT on Bernoulli's Differential Equation and Van der Pol's Oscillator with $N=33$ and $N=65$}
    \label{tab:physic_ode}
\end{table}
\begin{figure}[t]
\includegraphics[width=0.47\linewidth]{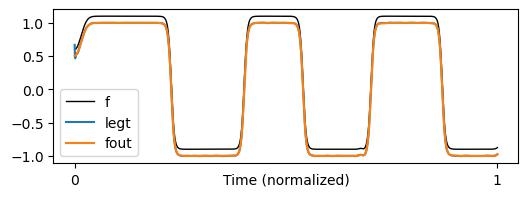}
  \hfill
  \includegraphics[width=0.47\linewidth]{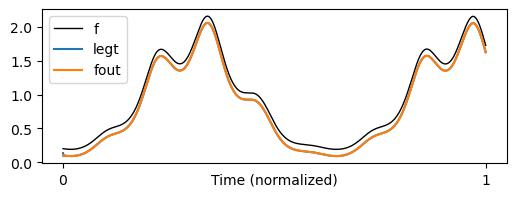}
     \caption{(Left) Solution and Prediction for Van der Pol's Oscillator with $N=65$.\\ (Right) Solution and Prediction for Bernoulli's Differential Equation with $N=65$}
     \label{fig:physics_ode}
\end{figure}
\subsection{Ablations on behavior of predictions for increasing state dimension}
In this section, we aim to validate the intuition that if we increase the state dimension $N$, the performance of our parametrization also increases as depicted in \ref{subfig:time_dependence}. This is because, as we have shown theoretically, we can express $\dot{u}$ as an infinite series of terms which we truncate after $N$ terms. As $N$ increases, one hence expects the approximation to get better. However, this only holds in the continuous setting. In the discrete setting instead, which we consider in our experiments, we are limited by the resolution with which we sample our input signal $u(t)$. 
\par
For this experiment, we take White Signal with $\gamma=1$, $2$ and $5$  and Filtered Noise with $\alpha=0.05$, $0.1$ and $0.3$ and compare the performance as $N$ increases for both LegT and FouT. We let $N$ vary from 1 to 96 with a stepsize of 5. In both White Signal and Filtered Noise, the performance of LegT rapidly decreases and then seems to stabilize. For FouT instead, the performance gets worse first and then decreases again with $N$. As we briefly mentioned in the main text, we hypothesize this is due to the fact that for low $N$ the models performs copying, i.e. it predicts the current timestep as the next one. If the signal is not too discountinuous, copying is a strategy that results in a low error since the current value of the signal and the next are fairly close. However, as $N$ increases, the model starts performing better than copying and it effectively manages to predict the next value of the input signal.

\subsection{Comparing different Initalizations}
\label{subsec:experiment3}
In this section, we provide some further information on the results obtained in \ref{subfig:construction_vs_training}. In the experiment, we use a Mixed Dataset to train the model consisting of: Sums of randomly drawn sines of frequencies between 0 and 50, Randomly Drawn White Signal's using a frequency uniformly drawn between $[0.3,1.5]$ and random Legendre polynomials up to degree 15. We empirically found this data mixture to be beneficial for the model to not overfit too much to a specific function class.
The model is then evaluated on Linear functions with slopes ranging from $[-10,10]$, the Van der Pol Function, the aforementioned Mixed Dataset and Filtered Noise functions.
Specifically, we consider three learning settings:
\begin{enumerate}
  \item Initialize the model paramaters $\overline{\A},\overline{\B}$ as proposed by HiPPO-LegT and randomly sample $\overline{\C}$ and $\overline{\D}$ from $\mathcal{N}(0,I)$. Training $\overline{\C}$ \& $\overline{\D}$ on next value predictions.
  \item Initialize the model paramaters $\overline{\A},\overline{\B}$ as proposed by HiPPO-LegT and initialize $\overline{\C}$ and $\overline{\D}$ to be our proposed $\overline{\C},\overline{\D}$. Training $\overline{\C}$ \& $\overline{\D}$ on next value predictions.
  \item Initialize the model paramaters $\overline{\A},\B$ as proposed by HiPPO-LegT and initialize $\overline{\C}$ and $\overline{\D}$ to be our proposed $\overline{\C},\overline{\D}$. Training all of $\overline{\A}$,$\overline{\B}$,$\overline{\C}$ \& $\overline{\D}$ on next value predictions.
\end{enumerate}
For all model we use $N=32$ as this showed a good trade-off between performance and efficiency. The models are trained on a batch-size of 128 performing 1000 epochs (8000 gradient steps). Our findings are the following:
The best generalizing overall model is our explicit weight construction along with initializing to our constructions of $\overline{\C},\overline{\D}$ and training on all parameters. Most importantly the model initialized with HiPPO $\overline{\A}, \overline{\B}$ and gaussian $\overline{\C},\overline{\D}$ performs much worse and seems to struggle in finding a weight construction that lets it adapt optimally to predicting the next signal value from only its previous observations. This suggests that our weight constructions could serve as an intialization scheme that could allow the model to better adapt to context.